\documentclass{article}

\usepackage{amssymb,amsfonts,amsmath,amscd,dsfont,mathrsfs}
\usepackage{graphicx,psfrag,epsfig,amsthm}
\usepackage{booktabs, multirow}
\usepackage{algorithm, algpseudocode}
\usepackage{bbding}
\usepackage{color}

\DeclareMathAlphabet{\mathpzc}{OT1}{pzc}{m}{it}

\usepackage{setspace}
\usepackage{amsmath,amssymb}
\usepackage{url}
\usepackage{enumerate}

\newtheorem{claim}{Claim}[section]
\newtheorem{lemma}[claim]{Lemma}

\newtheorem{assumption}{Assumption}

\newtheorem{theorem}{Theorem}
\newtheorem{proposition}[claim]{Proposition}

\def\<{\langle}
\def\>{\rangle}

\def\G{{\mathcal G}}

\def\eps{{\varepsilon}}

\def\id{{\rm I}}
\def\sT{{\sf T}}
\def\proj{{\sf P}^{\perp}}
\def\conv{{\rm conv}}
\def\cconv{\overline{\rm conv}}
\def\cH{{\cal H}}
\def\inv{{-1}}

\def\P{{\mathbb P}}
\def\prob{{\mathbb P}}

\def\E{{\mathbb E}} 
\def\Var{{\textup{Var}}}

\def\reals{\mathbb{R}}

\def\normal{{\sf N}}

\def\Cov{{\rm Cov}}
\def\Var{{\rm Var}}

\def\Tr{{\sf {Tr}}}
\def\tr{{\sf {Tr}}}
\def\hth{\hat{\theta}}

\def\htheta{\hat{\theta}}

\def\bbeta{\bar{\beta}}
\def\xtil{\tilde{x}}
\def\Ball{{\sf Ball}}
\def\ropt{r^{\rm opt}}
\def\OurStrategy{{\sc SmoothExplore }}
\def\BallStrategy{{\sc BallExplore }}
\def\cG{\mathcal{G}}
\def\d{{\mathrm{d}}}

\def\Var{{\rm Var}}

\def\ind{\mathbb{I}}
\newcommand\norm[1]{\lVert{#1}\rVert}

\newcommand\myeqref[1]{{Eq.\,\eqref{#1}}}

\def\F{{\mathcal F}}

\def\cX{{\cal X}}
\def\cF{{\cal F}}
\def\mle{{\, \preceq \,}}
\def\mge{{\, \succeq \,}}

\begin{document}

\title{Linear Bandits in High Dimension and Recommendation Systems}
\author{Yash~Deshpande\thanks{Y.~Deshpande is with the Department of Electrical 
Engineering, Stanford University}
~and~Andrea~Montanari\thanks{A.~Montanari is with the Departments of Electrical 
Engineering and Statistics, Stanford University}}

\maketitle

\begin{abstract}
A large number of online services provide automated recommendations to 
help users to navigate through a large collection of items. New items
(products, videos, songs, advertisements) are suggested on the basis of the
user's past history and --when available-- her demographic
profile. Recommendations have to satisfy the dual goal of helping the
user to explore the space of available items, while allowing the
system to probe the user's preferences.

We model this trade-off using linearly parametrized multi-armed
bandits, propose a policy and prove upper and lower bounds on the cumulative ``reward''
that coincide up to constants in the data poor (high-dimensional) regime. 
Prior work on linear bandits has focused on the data rich (low-dimensional)
regime and used cumulative ``risk'' as the figure of merit. For this data rich regime, we 
provide a simple modification for our policy that achieves near-optimal
risk performance under more restrictive assumptions on the geometry of the problem. 
We test (a variation of) the scheme used for
establishing achievability on the Netflix and MovieLens datasets and obtain good
agreement with the qualitative predictions of the theory we develop.
\end{abstract}

\section{Introduction}

Recommendation systems are a key technology  for navigating through
the ever-growing amount of data that is available on the Internet
(products, videos, songs, scientific papers, and so on). Recommended  items are chosen on the
basis of the user's past history and have to strike the right balance
between two competing objectives: 
\begin{description}
\item[Serendipity] i.e. allowing accidental pleasant discoveries.
This has a positive --albeit hard to quantify-- impact on user
experience, in that it naturally limits the recommendations  monotony.
It also has a quantifiable positive impact on the systems, by
providing fresh independent information about the user preferences.
\item[Relevance] i.e. determining recommendations which are most
  valued by the user, given her past choices.
\end{description}
While this trade-off is well understood by practitioners, as well as
in the data mining literature \cite{ScheinColdStart,ZhangDiversity,SlaneyDiversity}, rigorous and mathematical work
has largely focused on the second objective 
\cite{SJ03,SRJ05,CaR08,Gross09,CandesTaoMatrix,KMO09,KMO09noise,KoltchinskiiMatrixCompletion}. 
In this paper we address the first objective, building on recent work on linearly
parametrized bandits \cite{Dani08,RusTsi10,AbSze11}.

In a simple model, the system recommends items $i(1), i(2), i(3),\dots$
sequentially at times $t\in\{1,2,3,\dots\}$. The item index at time
$t$ is selected from a large set $i(t)\in [M]\equiv\{1,\dots,M\}$. Upon viewing (or reading,
buying, etc.) item $i(t)$, the user provides feedback $y_t$ to the
system. The feedback can be explicit, e.g. a one-to-five-stars rating,
or implicit, e.g. the fraction of a video's duration effectively
watched by the user. We will assume that $y_t\in\reals$, 
although more general types of feedback  also play an important role
in practice, and mapping them to real values is sometimes non-trivial.

A large body of literature has developed statistical methods to
predict the feedback that a user will provide on a specific item, given
past data concerning the same and other users (see the 
references above). A particularly
successful approach uses `low rank'  or `latent space' models.
These models postulate that the rating $y_{i,u}$ provided by user $u$ on
item $i$ is approximately given by the scalar product of two
feature vectors $\theta_u$ and  $x_i\in\reals^p$ characterizing,
respectively, the user and the item. In formulae
\begin{align*}
y_{i,u} = \<x_i,\theta_u\> + z_{i,u}\, ,
\end{align*}
where $\<a,b\> \equiv \sum_{i=1}^p a_ib_i$ denotes the standard scalar
product, and $z_{i,u}$ captures unexplained factors. The resulting
matrix of ratings $y = (y_{i,u})$ is well-approximated by a rank-$p$
matrix.  

The items feature vectors $x_i$ can be either constructed explicitly,
or derived from users' feedback using matrix factorization
methods. Throughout this paper we will assume that they have been
computed in advance using either one of these methods and are hence
given. We will use the shorthand $x_t=x_{i(t)}$ for the feature vector
of the item recommended at time $t$.

Since the items' feature  vectors are known in advance, distinct users
can be treated independently, and we will hereafter focus on a single users, with
feature vector $\theta$.
The vector $\theta$ can encode demographic
information known in advance or be computed from the user's feedback.
While the model can easily incorporate the former, we will focus on
the most interesting case in which no information is known in advance.

We are therefore led to consider the linear bandit model
\begin{align}
	y_t = \< x_t, \theta \> + z_t \, ,\label{eq:GeneralModel}
\end{align}
where, for simplicity, we will assume $z_t \sim \normal(0, \sigma^2)$
independent of $\theta$, $\{x_i\}_{i = 1}^t$ and
$\{z_i\}_{i = 1}^{t-1}$. At each time $t$, the recommender is given to
choose a item feature vector $x_t\in\cX_p\subseteq\reals^p$, with
$\cX_p$ the set of feature vectors of the available items. A recommendation policy
is a sequence of random variables $\{x_t\}_{t\ge 1}$, $x_t\in\cX_p$
wherein $x_{t+1}$ is a function of the past history $\{y_\ell,x_\ell\}_{1\le \ell\le
t}$  (technically, $x_{t+1}$ has to be measurable on $\cF_t \equiv \sigma(\{y_\ell,x_\ell\}_{\ell = 1}%
  ^t)$).
The system is rewarded at time $t$ by an amount equal to the user
appreciation  $y_t$, and we let $r_t$ denote the expected reward,
i.e. $r_t\equiv \E(\<x_t, \theta\>)$. 

As mentioned above, the same linear bandit problem was already studied
in several papers, most notably by Rusmevichientong and Tsitsiklis
\cite{RusTsi10}. The theory developed in that work, however, has two
limitations that are important in the context of recommendation systems. First, the
main objective of \cite{RusTsi10} is to construct policies with nearly optimal
`regret', and the focus is on the asymptotic behavior for $t$ large
with $p$ constant. In this limit the regret per unit time goes to $0$.
In a recommendation
system, typical dimensions $p$ of the latent feature vector are about 20
to 50 \cite{BK07long,Kor08long,KBV09}. 
If the vector $x_i$  include explicitly
constructed features, $p$ can easily become easily much larger.
 As a consequence, existing theory requires at least
$t\gtrsim 100$ ratings, which is unrealistic for many recommendation
systems and a large number of  users.

Second, the policies that have been analyzed in \cite{RusTsi10} are based on
an alternation of pure exploration and pure exploitation. In
exploration phases, recommendations are completely independent of the
user profile. This is somewhat unrealistic (and potentially harmful)
in practice because it would translate into a poor user experience. Consequently, 
we postulate the following desirable properties for a ``good'' policy:
\begin{enumerate}
\item \emph{Constant-optimal cumulative reward: } For all time
  $t$, $\sum_{\ell=1}^tr_{\ell}$ is within a constant factor of the
  maximum achievable reward.
\item \emph{Constant-optimal regret:} Let the maximum achievable reward
	be $r^{\rm opt} \equiv \sup_{x\in\cX_p} \<x,\theta\>$, then
	the `regret' $\sum_{\ell=1}^t (r^{\rm opt}-r_\ell)$ is
  within a constant of the optimal.
\item \emph{Approximate monotonicity:} For any $0\le t\le s$, we have
  $\prob\{\<x_s,\theta\>\ge  c_1r_t\} \ge c_2$ for $c_1, c_2$ as close
  as possible to $1$.  
\end{enumerate}

We aim, in this paper, to address the first objection in a fairly 
general setting. In particular, when $t$ is small, say a constant times
$p$, we provide matching upper and lower bounds for the cumulative
reward under certain mild assumptions on the set of arms $\cX_p$. Under
more restrictive assumptions on the set of arms $\cX_p$, our policy 
can be extended to achieve near optimal regret as well. Although we will
not prove a formal result of the type of Point 3, our policy is an excellent
candidate in that respect.  

The paper is organized as follows : in Section \ref{sec:Results} we formally 
state our main results. In Section \ref{sec:Rel} we discuss further related work.
Some explication on the assumptions we make on the 
set of arms $\cX_p$ is provided in Section \ref{sec:Assm}. In Section
\ref{sec:Num} we present numerical simulations of our policy 
on synthetic as well as realistic data from the Netflix and MovieLens
datasets. We also compare our results with prior work, and in particular with the
policy of \cite{RusTsi10}.  Finally, 
proofs are given in Sections \ref{sec:Proof} and \ref{sec:Proof2}.

\section{Main results}\label{sec:Results}
We denote by $\Ball(x;\rho)$ the Euclidean ball in $\reals^p$ with radius
$\rho$ and center $x \in \reals^p$. If $x$ is the origin, we omit this
argument and write $\Ball(\rho)$. Also, we
denote the identity matrix as $\id_p$. 

Our achievability results are based on the following assumption on the
set of arms $\cX_p$.
\begin{assumption} \label{assm:1}
Assume, without loss of generality, $\cX_p \in \Ball(1)$. 
We further assume that there exists a subset of arms $\cX'_p\subseteq \cX_p$ such that:
\begin{enumerate}
	\item For each $x \in \cX'_p$ there exists a distribution $\P_x(z)$ 
	supported on $\cX_p$ with $\E_x(z) = x$ and $\E_x(zz^\sT) 
	\mge (\gamma/p) \id_p$, for a constant $\gamma > 0 $. Here $\E_x(\cdot)$ 
	denotes expectation with respect to $\P_x$. 
	\item For all $\theta \in \reals^p$, $\sup_{x \in \cX'_p} \< x, \theta \> 
	\geq \kappa \norm{\theta}_2$ for some $\kappa >0$. 
	\end{enumerate}
\end{assumption}
Examples of sets satisfying Assumption \ref{assm:1} and further discussion
of its geometrical meaning are deferred to Section
\ref{sec:Assm}. Intuitively, it requires that $\cX_p$ is `well
spread-out'  in the unit ball $\Ball(1)$.

Following \cite{RusTsi10} we will also assume  $\theta \in \reals^p$
to be drawn from a  Gaussian prior $\normal(0, \id_p/p)$. 
This roughly corresponds to the assumption that nothing
is known a priori about the user except the length of its feature
vector $\norm\theta \approx 1$.
Under this assumption, the scalar product $\<x_1,\theta\>$, where $x_1$
is necessarily independent of $\theta$, is also
Gaussian with mean $0$ and variance $1/p$ and hence  $\Delta =
p\sigma^2$ is noise-to-signal ratio for the problem. Our results are
explicitly computable and apply to any value of $\Delta$. However they
are constant-optimal for $\Delta$ bounded away from zero. 

Let
$\htheta_t$ be the posterior mean estimate of $\theta$ at time
$t$, namely
\begin{eqnarray}
\htheta_t \equiv \arg\min_{\theta\in\reals^p}
\Big\{\frac{1}{2\sigma^2}\sum_{\ell=1}^{t-1}\big(y_\ell-\<x_{\ell},\theta\>\big)^2
+\frac{1}{2p}\|\theta\|^2\Big\}\, .
\end{eqnarray}
A greedy policy would select the arm $x\in\cX_p$ that maximizes the
expected one-step reward $\<x,\htheta_t\>$. As for the
classical multiarmed bandit problem, we would like to combine this
approach with random exploration of alternative arms.
We will refer to our strategy as \OurStrategy since it combines
exploration and exploitation in a continuous manner. 
This  policy is summarized in Table \ref{alg:se}.
\begin{algorithm}
	\caption{\OurStrategy}\label{alg:se}
\begin{algorithmic}[1]
\State \textbf{initialize} $\ell = 1$, $\htheta_1 = 0$, $\hth_1/\norm{\hth_1} = e_1$, $\Sigma_1 = \id_p/p$.  
 \Repeat 
 \State Compute: $\xtil_\ell = \arg \max_{x \in \cX'_p} \<\htheta_\ell, x\>$.
 \State Play: $x_\ell \sim \P_{\xtil_\ell}(\cdot)$, observe $y_t = \<x_t, \theta\> + z_t$. 
 \State Update: $\ell \gets \ell+1$, $\htheta_\ell = \arg\min_{\theta \in \reals^p} \frac{1}{2\sigma^2} %
 \sum_{i = 1}^{\ell-1} (y_i - \<x_i, \theta\>)^2 + \frac{1}{2p}\norm{\theta}^2$.   
\Until $\ell > t$
\end{algorithmic}
\end{algorithm}

The policy \OurStrategy uses a fixed mixture of exploration and
exploitation as prescribed by the probability kernel
$\prob_x(\,\cdot\,)$. As formalized below, this is constant optimal in
the data poor high-dimensional regime  hence on
small time horizons.

While the focus of this paper is on the data poor regime, 
it is useful to discuss how the latter blends with the data rich
regime that arises on long time horizons. This also clarifies where 
the boundary between short and long time horizons sits.  
Of course, one possibility would be to switch to a long-time-horizon
policy such as the one of \cite{RusTsi10}. Alternatively, in the spirit of
approximate monotonicity, we can try
to progressively reduce the random exploration component as $t$ increases.
We will illustrate this point  for the special case
$\cX_p\equiv\Ball(1)$. In that case, we 
introduce a  special case of \OurStrategy, called \BallStrategy, cf. Table 
\ref{alg:be}. The amount of random exploration at time $t$ is gauged
by a parameter $\beta_t$ that decreases from $\beta_1=\Theta(1)$ to
$\beta_t\to 0$ as $t\to\infty$. 

Note that, for $t\le p\Delta$,
$\beta_t$ is kept constant with $\beta_t=\sqrt{2/3}$. In  this regime
\BallStrategy corresponds to \OurStrategy with the choice
$\cX'_p=\partial\Ball(1/\sqrt 3)$ (here and below
$\partial S$ denotes the boundary of a set $S$). It is not hard to
check that this choice of $\cX'_p$ satisfies Assumption \ref{assm:1}
with $\kappa = 1/\sqrt 3$
and $\gamma = 2/3$. For further discussion on this point, we refer the reader to Section
\ref{sec:Assm}.

\begin{algorithm}
	\caption{\BallStrategy}\label{alg:be}
\begin{algorithmic}[1]
\State \textbf{initialize} $\ell = 1$, $\htheta_1 = 0$, $\hth_1/\norm{\hth_1} = e_1$, $\Sigma_1 = \id_p/p$, $\proj_1 = \id_p - e_1e_1^\sT$.  
 \Repeat 
 \State Compute: $\xtil_\ell = \arg \max_{x \in \Ball(1)} \<\htheta_\ell, x\> = \htheta_\ell/\norm{\htheta_\ell}$, $\beta_\ell = \sqrt{2/3}\min(p\Delta/\ell, 1)^{1/4}$. 
 \State Play: $x_t = \sqrt{1-\beta_\ell^2} \xtil_\ell + \beta_\ell \proj_\ell u_\ell$, where $u_\ell$ is a uniformly sampled unit vector, independent of the past. 
 \State Observe: $y_t = \<x_t, \theta\> + z_t$. 
 \State Update: $\ell \gets \ell+1$, $\htheta_\ell = \arg\min_{\theta \in \reals^p} \frac{1}{2\sigma^2}%
 \sum_{i = 1}^{\ell-1} (y_i - \<x_i, \theta\>)^2 + \frac{1}{2p}\norm{\theta}^2$, $\proj_\ell = \id_p - \htheta_\ell\htheta_\ell^\sT/\norm{\htheta_\ell}^2$.   
\Until $\ell > t$
\end{algorithmic}
\end{algorithm}
Our main result characterizes the cumulative reward
\begin{align*}
R_t \equiv \sum_{\ell=1}^tr_t = \sum_{\ell=1}^t
\E\{\<x_{\ell},\theta\>\}\, .
\end{align*}
\begin{theorem}\label{thm:optrew}
Consider the linear bandits problem with
$\theta\sim\normal(0,\id_{p\times p}/p)$, $x_t\in
\cX_p\subseteq\Ball(1)$ satisfying Assumption \ref{assm:1},
and $p\sigma^2=\Delta$. Further assume that $p\ge 2$ and $p\Delta \ge 2$. 

Then there exists a constant $C_1=C_1(\kappa, \gamma, \Delta)$ bounded for 
$\kappa, \gamma $ and $\Delta$  bounded away from zero, 
such that \OurStrategy achieves, for $1< t \leq p\Delta$, cumulative reward
\begin{align*}
 \quad R_t &\geq C_1\,
t^{3/2}p^{-1/2}. 
%
\end{align*}
%
Further, the cumulative reward of \emph{any} strategy is bounded for $1 \le t \le p\Delta$ as:
\begin{align*}
\quad	R_t\leq C_2\,
t^{3/2}p^{-1/2}\, .
\end{align*}
We may take the constants $C_1(\kappa, \gamma, \Delta)$ and $C_2(\Delta)$ to be:
\begin{align*}
	C_1 = \frac{\kappa \sqrt \Delta\, C(\gamma, \Delta)}{24\,\alpha(\gamma, \Delta)},%
&\qquad 	C_2 = \frac{2}{3\sqrt\Delta}, \\
\textrm{where } \quad C(\gamma, \Delta)= \frac{\gamma}{4(\Delta+1)},%
&\qquad	\alpha(\gamma, \Delta) = 1+\left[3\log\left( \frac{96}{\Delta\, C(\gamma, \Delta)}\right)\right]^{1/2}.
\end{align*}
\end{theorem}
In the special case where $\cX_p=\Ball(1)$, we have the 
following result demonstrating that \BallStrategy has near-optimal
performance in the long time horizon as well.
\begin{theorem}\label{thm:optrewlar}
Consider the linear bandits problem with 
$\theta\sim\normal(0, \id_{p\times p}/p)$ with the set 
of arms $\cX_p$ is the unit ball, i.e. $\Ball(1)$.  Assume, 
$p \ge 2$ and $p\Delta \ge 2$. Then
\BallStrategy achieves for all $t > p\Delta$:
\begin{align*}
	 R_t &\ge \ropt\, t - C_3%
	\, (pt)^{1/2+\omega(p)} \, .
\end{align*}
where:
\begin{align*}
\omega(p) = 1/(2(p+2)),&\qquad%
C_3(\Delta) = 70\left(\frac{\Delta+ 1}{\sqrt\Delta}\right).
\end{align*} 
\end{theorem}

For $t>p\Delta$, we can obtain a matching upper bound by a simple
modification of the arguments
in \cite{RusTsi10}.
\begin{theorem}[Rusmevichientong and Tsitsiklis]\label{thm:upperbndlar}
	Under the described model, the cumulative reward of any policy
        is bounded as
        follows 
\begin{align*}
\mbox{for }t > p\Delta, \quad 	R_t \leq \ropt\,  t - \sqrt{pt\Delta}
+ \frac{p\Delta}{2}\, .
\end{align*}
\end{theorem}

The above results characterize a sharp dichotomy between a
low-dimensional, data rich regime for $t>p\Delta$ and a
high-dimensional, data poor regime for $t\le p\Delta$. In the first
case classical theory applies: the reward approaches the oracle
performance with a gap of order $\sqrt{pt}$.
This behavior is in turn closely related to central limit theorem 
scaling in  asymptotic statistics.  Notice that the scaling with $t$
of our upper bound on the risk of \BallStrategy for large $t$ is suboptimal, namely
$(pt)^{1/2 + \omega (p)}$. Since however $\omega(p) = \Theta(1/p)$ the
difference can be seen only on exponential time scales $t\ge
\exp\{\Theta(p)\}$ and is likely to be irrelevant for moderate to
large values $p$ (see Section \ref{sec:Num} for a demonstration). It
is an open problem to establish the exact asymptotic
scaling\footnote{Simulations suggest that the upper bound $(pt)^{1/2 +
    \omega (p)}$ might be tight.} of \BallStrategy.

In the high-dimensional, data poor regime $t\le p\Delta$, the number of observations is
smaller than the number of  model parameters and the vector $\theta$ can only be
partially  estimated. Nevertheless, such partial estimate can be
exploited to produce a cumulative reward scaling as $t^{3/2}p^{-1/2}$.
 In this regime performances are not limited by central limit theorem
fluctuations in the estimate of $\theta$. The limiting factor is
instead the dimension of the parameter space that can be effectively
explored in $t$ steps. 

In order to understand this behavior, it is convenient to consider the
noiseless case $\sigma=0$. This is a somewhat degenerate case that, 
although not covered by the above theorem, yields useful
intuition. In the noiseless case, acquiring $t$ observations $y_1$,
\dots $y_t$ is equivalent to learning the projection of $\theta$ on the
$t$-dimensional subspace spanned by $x_1,\dots,x_t$. 
Equivalently, we learn $t$ coordinates of $\theta$ in a suitable
basis.  Since the mean square value of each component of $\theta$ is
$1/p$, this yields an estimate of $\htheta_t$ (the restriction to these
coordinates) with $\E\|\htheta_t\|^2_2 =t/p$. By selecting
$x_t$ in the direction of $\htheta_t$ we achieve instantaneous reward
$r_t\approx\sqrt{t/p}$
and hence cumulative reward $R_t = \Theta(t^{3/2}p^{-1/2})$ as stated
in the theorem.

\section{Related work}\label{sec:Rel}
Auer in \cite{Auer02} first considered a model similar to ours, wherein
the parameter $\theta$ and noise $z_t$ are bounded almost surely. This work
assumes $\cX _p$ finite and introduces an algorithm based on upper confidence
bounds. Dani et al. \cite{Dani08} extended the policy of \cite{Auer02} for
arbitrary compact decision sets $\cX _p$. For finite sets, \cite{Dani08} 
prove an upper bound on the regret that is logarithmic in its cardinality 
$\vert \cX_p \vert$, while for continuous sets they 
prove an upper bound of $O(\sqrt{pt}\log^{3/2}t)$. This result was further 
improved  by logarithmic factors in \cite{AbSze11}. The common theme throughout
this line of work is the use of upper confidence bounds and least-squares estimation.
The algorithms typically construct ellipsoidal confidence sets around the 
least-squares estimate $\hth$ which, with 
high probability, contain the parameter $\theta$. The algorithm then chooses optimistically 
the arm that appears the best with respect to this ellipsoid. 
As the confidence ellipsoids are initialized to be large, the bounds are 
only useful for $t \gg p$. In particular, in the high-dimensional
data-poor regime $t= O(p)$, the bounds typically become trivial.
In light of Theorem \ref{thm:upperbndlar} this is not surprising. 
Even after normalizing the noise-to-signal ratio while scaling the dimension, 
the $O(\sqrt{pt})$ dependence of the risk is relevant only for large 
time scales of $t \geq p\Delta$. This is the regime in which 
the parameter $\theta$ has been estimated fairly well.   

 Rusmevichientong and Tsitsiklis \cite{RusTsi10}  propose a phased 
policy which operates in distinct phases of learning the parameter $\theta$ 
and earning based on the current estimate of $\theta$. Although this approach 
yields order optimal bounds for the regret, it suffers from the same
shortcomings as confidence-ellipsoid based algorithms. In fact, \cite{RusTsi10} 
also consider a more general policy based on confidence bounds and prove a 
$O(\sqrt{pt}\log^{3/2}t)$ bound on the regret.

Our approach to the problem is significantly different and does not rely
on confidence bounds.    It would be interesting to understand whether
the techniques developed here can be use to improve the confidence
bounds method.

\section{On Assumption \ref{assm:1}}\label{sec:Assm}

The geometry of the set of arms $\cX_p$ is an important factor in the in the performance of any 
policy. For instance, \cite{RusTsi10}, \cite{Dani08} and \cite{AbSze11}
provide ``problem-dependent'' bounds on the regret incurred in terms
of the difference between the reward of the optimal
arm and the next-optimal arm. 
This characterization is reasonable in the long time horizon:
 if the posterior estimate $\htheta_t$ of the feature vector $\theta$ \emph{coincided} with
$\theta$ itself, only the optimal arm would matter. 
Since the posterior estimate converges to
$\theta$ in the limit of large $t$, the local geometry of $\cX_p$
around the optimal arm dictates the asymptotic behavior of the regret.

In the high-dimensional, short-time regime, the global geometry of
$\cX_p$  plays instead a crucial role. This is quantified in our
results through the parameters $\kappa$ and $\gamma$ appearing in
Assumption  \ref{assm:1}. Roughly speaking, this amounts to requiring
that $\cX_p$ is `spread out' in the unit ball.  It is useful to
discuss this intuition in a more precise manner. For the proofs of 
statements in this section we refer  to Appendix \ref{app:Arms}.

A simple case is the one in which the arm set contains a ball.
\begin{lemma}\label{lemma:ContainsBall}
If $\Ball(\rho)\subseteq \cX_p\subseteq\Ball(1)$, then $\cX_p$
satisfies Assumption \ref{assm:1} with $\kappa=\rho/\sqrt{3}$, $\gamma=2\rho^2/3$. 
\end{lemma}

The last lemma does not cover the interesting case in which $\cX_p$ is
finite. The next result shows however that, for Assumption
\ref{assm:1}.2 to hold it is sufficient that the closure of the convex
hull of $\cX'_p$, denoted by $\cconv(\cX'_p)$,  contains a ball.
\begin{proposition}\label{prop:HullContainsBall}
Assumption \ref{assm:1}.2 holds if and only if
$\Ball(\kappa)\subseteq \cconv(\cX'_p)$.
\end{proposition}
In other words, Assumption  \ref{assm:1}.2 is satisfied if $\cX'_p$ is
`spread out' in all directions around the origin.

Finally, we consider a concrete example with $\cX_p$ finite. 
Let $x_1,x_2,\dots,x_M$ to be i.i.d. uniformly random in $\Ball(1)$. We then refer to the set of
arms $\cX_p \equiv \{x_1,x_2,\dots,x_M\}$ as to a
\emph{uniform cloud}.
\begin{proposition}\label{prop:UniformCloud}
A uniform  cloud $\cX_p$ in dimension $p \ge 20$ satisfies Assumption \ref{assm:1} with
$M = 8^p$, $\kappa=1/4$ and $\gamma=1/32$ with probability larger than $1-2\exp(-p)$.
\end{proposition}

\section{Numerical results}\label{sec:Num}

\begin{figure}[t]
\includegraphics[width=0.47\textwidth]{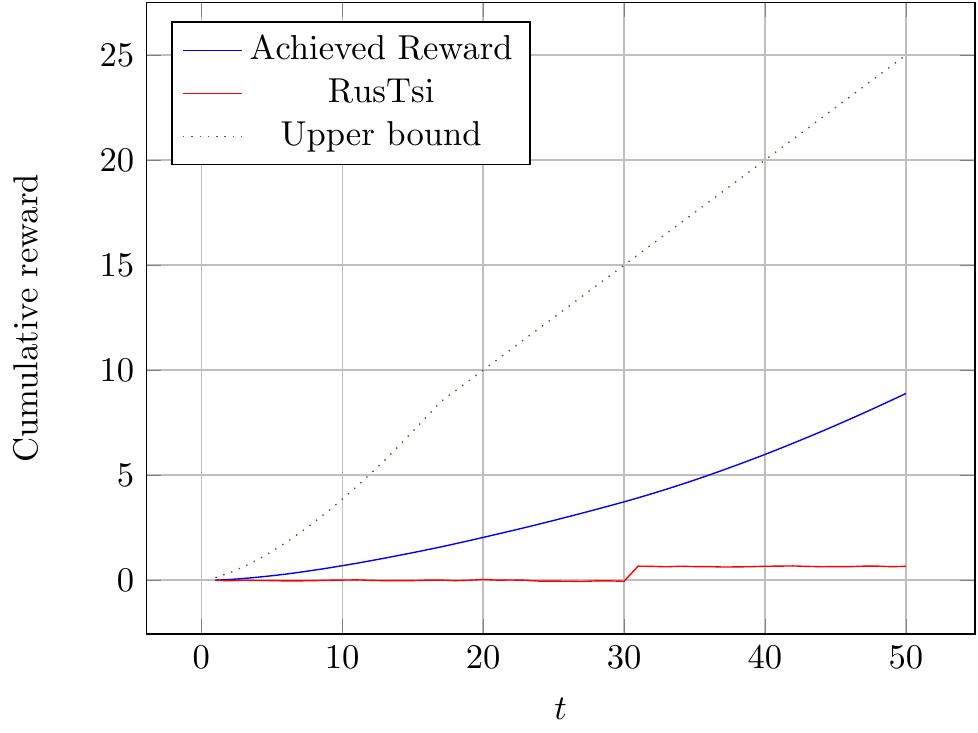}
\includegraphics[width=0.47\textwidth]{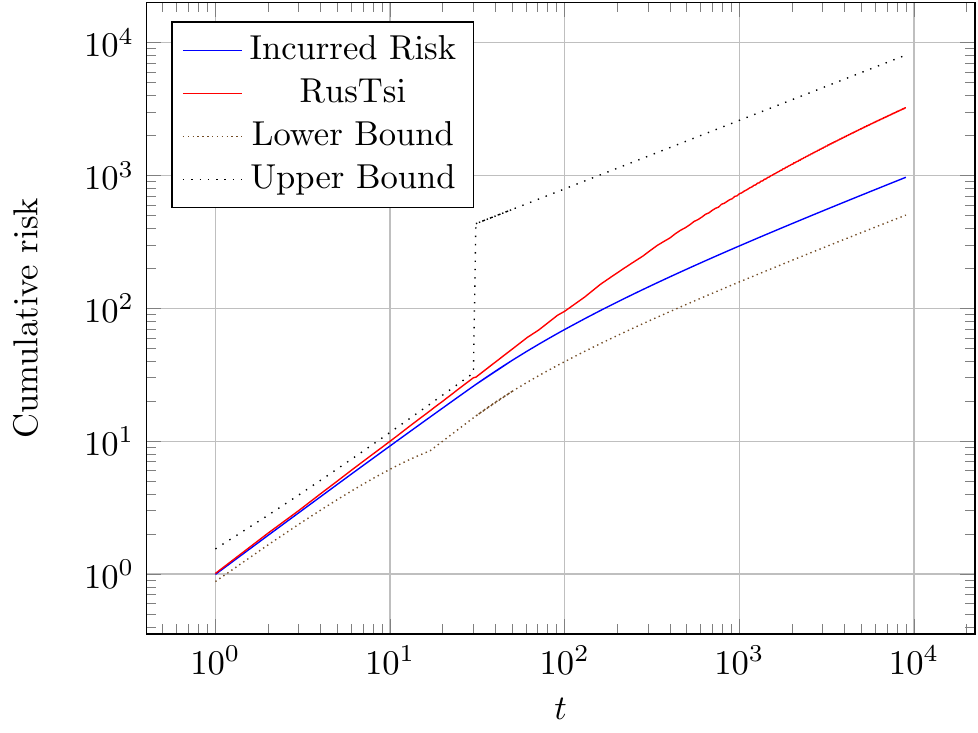}
\caption{Left frame: Cumulative reward $R_t$ in the data poor regime
  $t\lesssim 2\,p\Delta$ (here $p=30$, $\Delta=1$) as obtained through numerical simulations over
  synthetic data, together with analytical upper bound. Right frame:
  Cumulative risk in the data rich regime $t\gg p\Delta$ (again, $p=30$, $\Delta=1$).}\label{fig:reward}
\end{figure}

We will mainly compare our results with those of \cite{RusTsi10} since
the results of that paper directly apply to the present problem. The authors proposed a phased 
exploration/exploitation policy, wherein they separate the phases of 
learning the parameter $\theta$ (exploration) and earning reward based on the 
current estimate of $\theta$ (exploitation). 

In Figure \ref{fig:reward} we plot
the cumulative reward and the cumulative risk incurred by our policy and the 
phased policy, as well as analytical bounds thereof. 
We  generated $\theta\sim \normal(0,\id_{p})$  randomly for
$p=30$, and produced observations $y_t$, $t\in\{1,2,3,\dots\}$
according to the general model
(\ref{eq:GeneralModel}) with $\Delta= p\sigma^2=1$ and arm set $\cX_p=\Ball(1)$.
The curves presented here are averages over $n=5000$ realizations and
statistical fluctuations are negligible.

The left frame illustrates the performance of \OurStrategy  in the data
poor (high-dimensional) regime $t\lesssim 2\, p\Delta$. 
We compare the cumulative reward $R_t$ as achieved in simulations, 
with that of the phased policy of \cite{RusTsi10}
and  with the theoretical upper bound of Theorem \ref{thm:optrew} (and
Theorem \ref{thm:upperbndlar} for $t>p\Delta$). 
In the right frame  we consider instead the data rich
(low-dimensional) regime $t\gg p\Delta$. In this case it is more
convenient to plot the cumulative risk $t\ropt-R_t$.
We plot the curves corresponding to the ones in the left frame, as well as
the upper bound (lower bound on the reward) from Theorems \ref{thm:optrew}
and \ref{thm:optrewlar}.

Note that the $O(\sqrt{pt})$ behavior of the  risk of the phased policy 
can be observed only for $t\gtrsim 1000$. On the other hand, our policy displays the 
correct behavior for both time scales. The extra $\omega (p) = \Theta(1/p)$ factor in the
exponent yields a multiplicative factor larger than $2$ only for $t \ge
2^{2(p+2)}\approx 2\cdot 10^{19}$.

\begin{figure}[t]
	\begin{center}
		\includegraphics[width=0.46\textwidth]{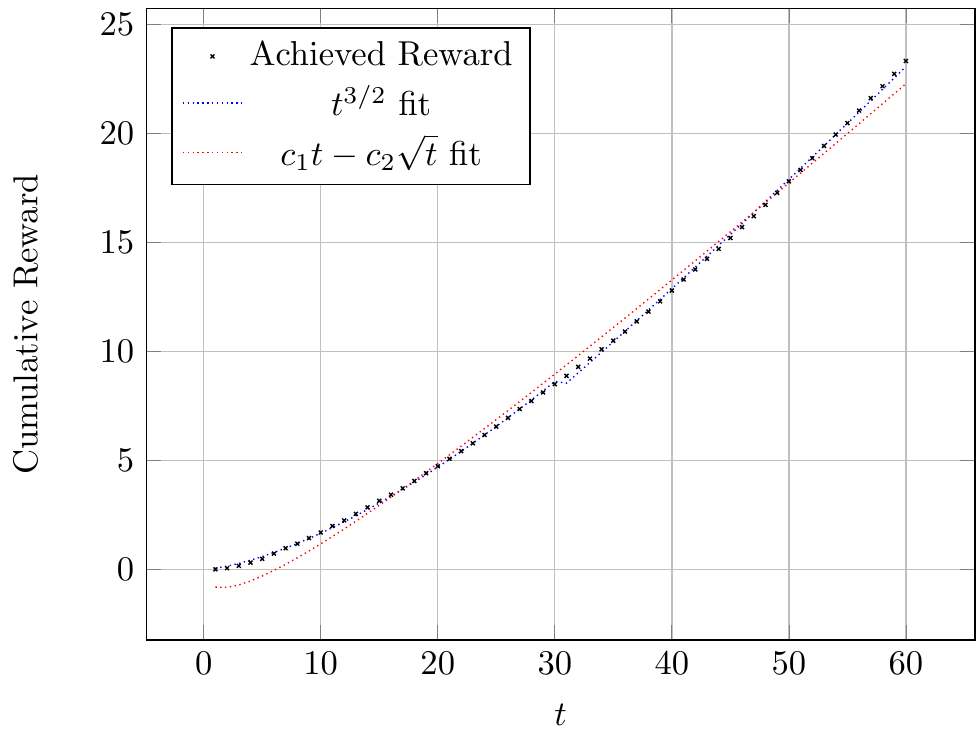}
		\hfill
		\includegraphics[width=0.46\textwidth]{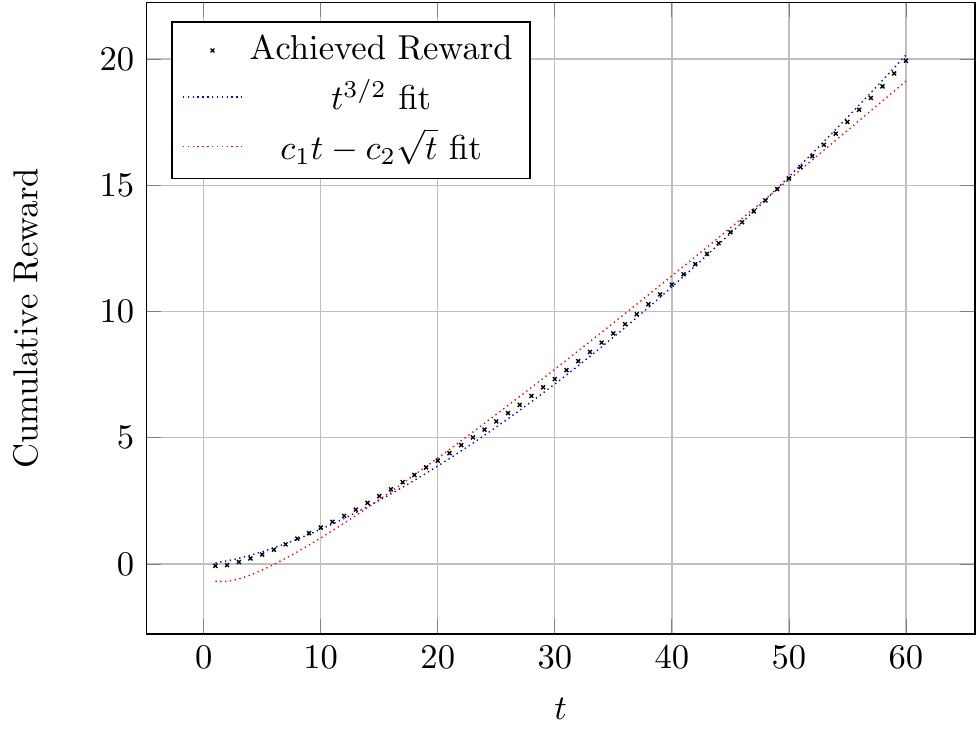}
	\end{center}
	\caption{Results using the Netflix (left frame) and MovieLens 1M (right frame) datasets.
          \OurStrategy is effective in learning the user's preferences
	  and is well described by the predicted behavior of Theorem \ref{thm:optrew}.}
	\label{fig:netflshort}
\end{figure}

The above set of numerical experiments used $\cX_p=\Ball(1)$. For
applications to recommendation systems, $\cX_p$ is in correspondence
with a certain catalogue of achievable products or contents. In
particular, $\cX_p$ is expected to be finite. It is therefore
important to check how does \OurStrategy perform for a  realistic sets of
arms.  We plot results obtained with the Netflix Prize dataset and the MovieLens 1M dataset in Figure \ref{fig:netflshort}. Here the 
feature vectors $x_i$'s for movies are obtained using the matrix
completion algorithm of \cite{KMO09noise}. The user parameter vectors $\theta_u$
were obtained by regressing the rating against the movie feature
vectors (the average user rating $a_u$ was subtracted). Similar to synthetic data, we took $p=30$. 
Regression also yields an estimate for the noise variance which is assumed known 
in the algorithm. 
We then simulated an interactive scenario by postulating that the
rating of user $u$ for movie $i$ is given by
\begin{align*}
\widetilde{y}_{i,u} = {\sf Quant}(a_u+\<x_,\theta_u\>)\, ,
\end{align*}
where ${\sf Quant}(z)$ quantizes $z$ to to $\{1, 2, \cdots, 5\}$
(corresponding to a one-to-five star rating).  
The feedback used for our simulation is the centered rating  $y_{i,u} =
\widetilde{y}_{i,u}-a_u$.

We implement a slightly modified version of \OurStrategy for these
simulations. At each time we compute the ridge regression estimate
of the user feature vector $\htheta_t$ as before and choose the ``best'' movie
$\tilde x_t = \arg \max _{x \in \cX_p } \<x, \htheta_t\>$ assuming 
our estimate is error free. We then construct the ball in
$\reals^p$ with center $\tilde x_t$ and radius
$\beta_t$. We list all the movies whose feature vectors fall in this
ball, and recommend a uniformly randomly chosen one in this list. 

Classical bandit theory implies the reward behavior is of the type $c_1 t - c_2 \sqrt{t}$ 
where $c_1$ and $c_2$ are (dimension-dependent) constants. Figure
\ref{fig:netflshort} presents the best fit of this type for $t\lesssim
2p$. The description appears to be qualitatively incorrect in this regime.
Indeed, in this regime,  the reward behavior is better explained by a $c_3t^{3/2}$ curve.
These results suggest that our policy is fairly robust to the significant
modeling uncertainty inherent in the problem. In particular, the fact that
the ``noise'' encountered in practice is manifestly non-Gaussian does not
affect the qualitative predictions of Theorem \ref{thm:optrew}.

A full validation of our approach would require an actual interactive
realization of a recommendation system
\cite{OurImplementationInteractive}.
Unfortunately, such validation cannot be provided by existing
datasets, such as the ones used here.
A naive approach would be to use the actual ratings as the feedback $y_{iu}$, but 
this suffers from many shortcomings.
First of all, each user rates a sparse  subset (of the order of $100$ movies) of the
whole database of movies, and hence any policy to be tested  would be
heavily constrained and distorted. Second, the set of rated movies is a biased subset (since it is selected
by the user itself).

\section{Proof of Theorem \ref{thm:optrew}}\label{sec:Proof}
We begin with some useful notation. Define the $\sigma$-algebra 
$\cF_t \equiv \sigma(\{y_\ell, x_\ell\}_{\ell = 1}^t)$. Also let 
$\cG_t \equiv \sigma(\{y_\ell\}_{\ell=1}^{t-1}, \{x_\ell\}_{\ell=1}^t)$.
We let $\htheta_t$ and $\Sigma_t$ denote the posterior mean and covariance of $\theta$
given $t-1$ observations. Since $\theta$ is Gaussian and the 
observations are linear, it is a standard result that these can 
be computed as:
\begin{align*}
	\Sigma_t &\equiv \Cov(\theta\vert\cF_{t-1}) = \left(p\id_p + \frac{1}{\sigma^2}\sum_{\ell=1}^{t-1} x_\ell x_\ell^\sT\right)^{-1}\\ 
	\htheta_t &\equiv \E(\theta\vert\cF_{t-1}) =  \Sigma_t\left( \sum_{\ell=1}^{t-1}\frac{ y_\ell}{\sigma^2} x_\ell \right).
\end{align*}
Note that since $\theta$ is Gaussian and the measurements are 
linear the posterior mean coincides with the maximum likelihood
estimate for $\theta$. This ensures our notation is consistent.
\subsection{Upper bound on reward} 

At time $\ell$, the expected reward $r_\ell = \E(\<x_\ell, \theta\>) \le \E(\norm{\htheta_\ell}) \leq  
\bigl[\E(\norm{\htheta_\ell}^2)\bigr]^{1/2}$, where the first inequality follows 
from Cauchy-Schwarz, that $\htheta_\ell$ is unbiased and that $\norm{x_\ell} \le 1$. 
Since $1 = \E(\norm\theta^2) = \E(\norm{\htheta_\ell}^2) + \E(\tr{\Sigma_\ell})$: 
\begin{align}
	r_\ell^2 \leq 1 - \E \left(\tr(\Sigma_\ell) \right).
	\label{eqn:rew}
\end{align}
	We have, applying Jensen's inequality and further simplification:
\begin{align*}
	\E\tr(\Sigma_\ell) &\ge p^2/\E(\tr(\Sigma_\ell^{-1})) \\
	&= p^2/\E\tr\left(p\id_p + \frac{1}{\sigma^2}\sum_{j=1}^{\ell-1}x_j x_j^\sT \right) \\
	&\ge \left(1 + \frac{\ell-1}{p^2\sigma^2}\right)^{-1}.
\end{align*}
Using this to bound the right hand side of \myeqref{eqn:rew}
\begin{align*}
		r_\ell^2 &\leq 1 - \frac{1}{1+ (\ell-1)/(p\sigma)^2} \\
		&= \frac{(\ell-1)/p}{(\ell-1)/p + {p\sigma^2}} \\
		&\leq \frac{1}{p\sigma^2}\frac{\ell-1}{p}.
\end{align*}
	The cumulative reward can then be bounded as follows:
\begin{align*}
	\sum_{\ell=1}^{t}r_\ell &\le \frac{1}{\sqrt{p\sigma^2}} \sum_{\ell=1}^{t} \sqrt{\frac{\ell-1}{p}} \\
	&\leq \frac{2}{3\sqrt{p\sigma^2}} t^{3/2}p^{-1/2}\\
	&= C_2(\Delta)t^{3/2}p^{-1/2}.
\end{align*}
Here we define $C_2(\Delta) \equiv 2/3\sqrt\Delta$. 
\subsection{Lower bound on reward}
We compute the expected reward earned by \OurStrategy at time $t$ as:

\begin{align}\label{eqn:exprew}
	r_t &= \E(\<x_t, \theta\>) \nonumber\\
	&= \E(\E(\<x_t, \theta\> \vert \cG_{t-1})) \nonumber\\
	&= \E(\E(\<x_t, \htheta_t\> \vert \cG_{t-1})) \nonumber\\
	&= \E(\<\xtil_t, \htheta_t\>) \nonumber\\
	&\ge \kappa\, \E(\norm{\htheta_t}).
\end{align}

The following lemma guarantees that $\norm{\htheta_t}$ is $\Omega(\sqrt{t})$. 

\begin{lemma}\label{lem:1stmom}
	Under the conditions of Theorem \ref{thm:optrew} we have, for all $t > 1$:
	\begin{align*}
		\E\norm{\htheta_t} &\ge C'(\gamma, \Delta) t^{1/2}p^{-1/2}.
	\end{align*}
	Here:
	\begin{align*}
		C'(\gamma, \Delta) &= \frac{1}{2} \frac{C(\gamma, \Delta)}{\alpha(\gamma, \Delta)}\sqrt\frac{\Delta}{8},  \\%
		\textrm{where }\quad	C(\gamma, \Delta) &= \frac{\gamma}{4(\Delta+1)} \\
		\alpha(\gamma, \Delta) &= 1 + \left[ 3 \log \left( \frac{96}{\Delta C(\gamma, \Delta)} \right)\right]^{1/2}.
	\end{align*}
\end{lemma}
Using this lemma we proceed by bounding the right side of \myeqref{eqn:exprew}:
\begin{align*}
	r_t &\geq \kappa C'\sqrt\frac{t-1}{p}.
\end{align*}
Computing cumulative reward $R_t$ we have:
\begin{align*}
 R_t  &= \sum_{\ell=1}^{t} r_\ell \\ 
 &\geq \sum_{\ell = 1}^{t}\kappa C'\sqrt\frac{\ell-1}{p} \\
 &\geq \kappa C'\int_0^{t-1}\sqrt{\frac{\nu}{p}}\mathrm{d} \nu  \\
 &\geq \frac{2}{3}\kappa C' (t-1)^{3/2}p^{-1/2} \\
 &\ge \frac{\kappa C'}{3\sqrt{2}} t^{3/2}p^{-1/2}.
\end{align*}

Thus, letting $C_1(\kappa, \gamma, \Delta) = \kappa C'(\gamma, \Delta)/3\sqrt 2$, 
we have the required result.
\subsection{Proof of Lemma \ref{lem:1stmom}}
In order to prove that $\E(\norm{\htheta_t}) = \Omega(\sqrt{t})$, we will first 
show that $\E(\norm{\htheta_t}^2)$ is $\Omega(t)$. Then we prove that $\norm{\htheta_t}$ 
is sub-gaussian, and use this to arrive at the required result. 

\begin{lemma}[Growth of Second Moment]\label{lem:2ndmom}
Under the conditions of Theorem \ref{thm:optrew}:	
	\begin{align*}
		\E{\norm{\hth_t}^2} \geq C(\Delta, \gamma) \frac{t-1}{p},	
	\end{align*}
	where
	\begin{align*}
		C(\Delta, \gamma) = \frac{\gamma}{4(\Delta+1)}. 
	\end{align*}
\end{lemma}
\begin{proof}
We rewrite $\htheta_t$ using the following inductive form:
\begin{align}\label{eqn:indform}
	\htheta_{t+1} = \htheta_t + \Sigma_{t+1}\left( \frac{1}{\sigma^2} x_t x_t^\sT \right) v_t + \Sigma_{t+1} \frac{z_t}{\sigma^2} x_t.
\end{align}
Here $v_t \equiv \theta - \hth_t$ is a random zero mean vector. Conditional on
$\cF_{t-1}$, $v_t$ is  distributed as $\normal(0, \Sigma_t)$ and is independent 
of $x_t$ and $z_t$. Recall that the  $\sigma$-algebra $\cG_t = \sigma(\{y_\ell\}_{\ell=1}^{t-1},%
\{x_\ell\}_{\ell=1}^{t}) \supseteq \cF_{t-1}$. 
Then we have:
\begin{align}
	\E(\norm{\hth_{t+1}}{}^2 \vert \G_t) &= \norm{\hth_t}{}^2 + \frac{1}{\sigma^4}%
	\E \left[ v_t^\sT \left( \Sigma_{t+1}x_tx_t^\sT \right)^\sT \left( \Sigma_{t+1}x_tx_t^\sT \right)%
	v_t \big\vert \G_t \right] + \frac{1}{\sigma^4}\E \left[ z_t^2 \vert \G_t \right]%
	\left(\Sigma_{t+1}x_t \right)^\sT\left(\Sigma_{t+1}x_t\right) .
	\label{eqn:secondmom1}
\end{align}
The cross terms cancel since $v_t$ and $ z_t$ conditionally on $\G_t$ are 
independent and zero mean. The expectation in the second term can be reduced as follows:
\begin{align*}
	\E \left[ v_t^\sT \left( \Sigma_{t+1}x_tx_t^\sT \right)^\sT \left( \Sigma_{t+1}x_tx_t^\sT \right) v_t \big\vert \G_t \right] 
	& = \Tr \left[ \left(\Sigma_{t+1}x_tx_t^\sT \right) \Sigma_t \left( \Sigma_{t+1}x_tx_t^\sT \right)^\sT \right]\\
	&=  \Tr \left[ \left(\Sigma_{t+1}x_tx_t^\sT \right) \Sigma_t \left( x_t x_t^\sT\Sigma_{t+1}\right) \right]\\
	& = (x_t^\sT\Sigma_t x_t) \Tr \left[\Sigma_{t+1}x_tx_t^\sT \Sigma_{t+1} \right] \\
	&= \left(x_t^\sT\Sigma_t x_t \right) \left( x_t^\sT \Sigma_{t+1}^2 x_t \right).
\end{align*}
The third term can be seen to be:
\begin{align*}
	\E \left[ z_t^2 \vert \G_t \right] \left(\Sigma_{t+1}x_t \right)^\sT\left(\Sigma_{t+1}x_t\right)
	&= \sigma^2 x_t^\sT\Sigma_{t+1}^2 x_t.
\end{align*}
Thus we have, continuing \myeqref{eqn:secondmom1}:
\begin{align}
	\E(\norm{\hth_{t+1}}{}^2 \vert \G_t) &= \norm{\hth_t}{}^2 + %
	\frac{1}{\sigma^4}\left(\sigma^2 + x_t^\sT \Sigma_t x_t \right)\left( x_t^\sT \Sigma_{t+1}^2 x_t \right). \label{eqn:secondmom2}
\end{align}
Since $\Sigma_{t+1} = \left(\Sigma_t^{-1} + \frac{1}{\sigma^2}x_tx_t^\sT\right)^{-1} = %
\Sigma_t -   \Sigma_tx_tx_t^\sT\Sigma_t/ (\sigma^2 + x_t^\sT\Sigma_tx_t)$, some calculation yields that:
\begin{align*}
	x_t^\sT\Sigma_{t+1}^2x_t = \frac{\sigma^4 \left(x_t^\sT \Sigma_t^2 x_t \right)}%
	{\left(\sigma^2 + x_t^\sT\Sigma_tx_t\right)^2}.
\end{align*}
Thus \myeqref{eqn:secondmom2} reduces to 
\begin{align}
	\E(\norm{\hth_{t+1}}{}^2\vert \G_t) &= \norm{\hth_t}{}^2 + %
	\frac{x_t^\sT\Sigma_t^2x_t}{\sigma^2 + x_t^\sT\Sigma_tx_t}.
	\label{eqn:secondmom3}
\end{align}
We now bound the additive term in \myeqref{eqn:secondmom3}. 
We know that $\Sigma_t \mle  \id/p$ (the prior covariance), thus 
$x_t^\sT\Sigma_t x_t \leq 1/p$ since $x_t \in \cX_p \subseteq \Ball(1)$. 
Hence the denominator in \myeqref{eqn:secondmom3} is upper bounded by 
$\sigma^2 + 1/p$. To bound the numerator:
\begin{align*}
	\E[ x_t^\sT\Sigma_t^2 x_t \vert \F_{t-1} ]  &= %
	\E[\Tr(\Sigma_t^2 x_t x_t^\sT)\vert \F_{t-1} ] \\
	&= \Tr[\Sigma_t^2\E(x_t x_t^\sT \vert \F_{t-1})] \\
	&\ge \frac{\gamma}{p} \Tr(\Sigma_t^2), 
\end{align*}
since $\E_{\xtil_t}(x_t x_t^\sT) \mge (\gamma/p) \id_p$ by Assumption \ref{assm:1}.
Using this in \myeqref{eqn:secondmom3}, we take expectations to get:
\begin{align}
	\E(\norm{\hth_{t+1}}{}^2) &\geq \E(\norm{\hth_t}{}^2) + \frac{\gamma}{\Delta +1}\E[\Tr(\Sigma_t^2)].
	\label{eqn:secondmom4}
\end{align}
Considering the second term in \myeqref{eqn:secondmom4}:
\begin{align*}
	\E[\Tr(\Sigma_t^2)] &\geq p\,\E[\det(\Sigma_t^2)^{1/p}]\\
	&= p\, \E\left[ \left( \prod_{j = 1}^p \frac{1}{p+\lambda_j/\sigma^2}\right)^{2/p} \right],
\end{align*}
where $\lambda_j$ is the $j^{\text{th}}$ eigenvalue of 
$\sum_{\ell = 1}^{t-1} x_{\ell}x_{\ell}^\sT$. Continuing 
the chain of inequalities:
\begin{align*}
	\E[\Tr(\Sigma_t^2)] &\geq \frac{1}{p}\, \E\left[ \prod_{j=1}^{p}\left(1+\frac{\lambda_j}{\Delta}\right)^{-2/p} \right]\\
	&\geq \frac{1}{p}\,\E\left[ \prod_{j=1}^{p}\exp\left(-\frac{2\lambda_j}{p\Delta}\right) \right] \\
	&= \frac{1}{p}\, \E \left[ \exp\left\{-\frac{2}{p\Delta}\Tr\Big(\sum_{\ell = 1}^t x_\ell x_\ell ^\sT\Big)\right\} \right] \\
	&\geq \frac{1}{p} \exp\left\{-\frac{2(t-1)}{p\Delta}\right\},
\end{align*}
where the last inequality follows from the fact that 
$x_\ell \in \Ball(1)$ for each $\ell$. Combining this with 
\myeqref{eqn:secondmom4} gives:
\begin{align}
	\E(\norm{\hth_{t+1}}{}^2 &\geq \E(\norm{\hth_t}^2) + \frac{\gamma}{\Delta + 1} \frac{1}{p} \exp\left\{{-\frac{2(t-1)}{p\Delta}}\right\}.
	\label{eqn:secondmom5}
\end{align}
Summing over $t$ this implies:
\begin{align*}
	\E[\norm{\hth_t}^2] &\geq \frac{\gamma}{p(\Delta+1)}\frac{1-\exp\{2(t-1)/p\Delta\}}{1-\exp\{-2/p\Delta\}} \\
	&\geq \frac{\gamma\Delta}{2(\Delta +1)} (1-\exp\{-2(t-1)/p\Delta\}) \\
	&\geq \frac{\gamma }{2(\Delta+1)} \left( 1 - \exp\{-2(p\Delta-1)/p\Delta\} \right) \left( \frac{t-1}{p}\right) .
\end{align*}
The last inequality follows from fact that $1-\exp(-z)$ is concave in $z$. 
Using $p\Delta \geq 2$, we obtain: 
\begin{align*}
	\E(\norm{\hth_t}{}^2) &\geq \frac{\gamma (1-e^{-1})}{2(\Delta+1)} \frac{t-1}{p}\\
	&\geq \frac{\gamma}{4(\Delta+1)} \frac{t-1}{p}.
\end{align*}
\end{proof}

\begin{lemma}[Sub-Gaussianity of $\norm{\hth_t}$] \label{lem:subgauss}
	Under the conditions of Theorem \ref{thm:optrew}
\begin{align*}
	\P\left(\norm{\htheta_t}\ge \sqrt\frac{8(t-1)}{p\Delta} \nu \right) \le e^{-(\nu-1)^2/3}.
\end{align*}
\end{lemma}
\begin{proof}
Note that $\htheta_t$ is a (vector-valued) martingale. The associated
difference sequence given by (cf. \myeqref{eqn:indform})
\begin{align*}
	\xi_{t} &= \frac{\<v_t, x_t\> + z_t}{\sigma^2}\Sigma_{t+1} x_t .
\end{align*}
Note that $\htheta_t = \sum_{\ell = 1}^{t-1} \xi_\ell$. 
We have that $\E(\xi_t\vert \cF_{t-1}) = 0$. 
Then conditionally on $\cG_{t}$, $\norm{\xi_t} = \vert w_t \vert %
\frac{\Sigma_{t+1}x_t}{\norm{\Sigma_{t+1}x_t}}$, where $w_t \equiv %
\frac{\<v_t, x_t\> + z_t}{\sigma^2}\norm{\Sigma_{t+1}x_t}$ is Gaussian 
with variance given by:

\begin{align*}
	\Var(w_t \vert \cG_t) &= \frac{\sigma^2 + x_t^\sT \Sigma_t x_t}{\sigma^4}x_t^\sT\Sigma^2_{t+1}x_t \\
	&= \frac{x_t^\sT \Sigma_t^2x_t}{\sigma^2 + x_t^\sT\Sigma_t x_t} \\ 
	&\le \frac{1}{p\Delta},
\end{align*}
since $0 \mle \Sigma_t \mle \id/p$ and $\norm{x_t}\le 1$. Thus, we have
the following ``light-tail'' condition on $\xi_t$:
\begin{align*}
	\E(e^{\lambda\norm{\xi_t}^2}\vert \cG_t) &\le \left( 1 - \frac{2\lambda}{p\Delta}\right)^{-1/2}.  
\end{align*}
Using $\lambda = p\Delta/4$, we obtain:
\begin{align*}
	\E(e^{p \Delta\norm{\xi_t}^2 /4}\vert \cG_t) &\le \sqrt{2} \le e. 
\end{align*}
Now using Theorem 2.1 in \cite{JudNem08} we obtain that: 
\begin{align*}
	\P\left(\norm{\htheta_t} \ge \sqrt\frac{8(t-1)}{p\Delta} (1+ \nu)\right) &\le e^{-\nu^2/3},
\end{align*}
which implies the lemma.
\end{proof}

We can now prove Lemma \ref{lem:1stmom}. We have:
\begin{align}
	\E[\norm{\hth_t}{}^2] &= \E\bigl[\norm{\hth_t{}}{}^2 \ind_{\norm{\hth_t{}}{}^2 \leq a} \bigr] %
	+ \E \bigl[\norm{\hth_t{}}{}^2 \ind_{\norm{\hth_t{}}{}^2 \geq a} \bigr] \nonumber \\
	&\leq \sqrt{a} \E[\norm{\hth_t}{}] + \int_{a}^{\infty} \P(\norm{\hth_t}{}^2 \geq y) \mathrm{d}y. 
	\label{eqn:1stmom1}
\end{align}
Here we use the fact that $\norm{\hth_t}{}$ is a positive random variable.
Employing Lemma \ref{lem:subgauss} to bound the second term:
\begin{align*}
	\int_a^{\infty} \P(\norm{\hth_t}{}^2 \geq y) \mathrm{d}y &\leq %
	\frac{8(t-1)}{p\Delta}\int_{\alpha}^\infty 2\nu e^{-(\nu-1)^2/3} \mathrm{d} \nu, \\
	&= \frac{8(t-1)}{p\Delta} \left( \int_\alpha ^\infty 2(\nu-1) e^{-(\nu-1)^2/3}\d\nu %
	+ 2\int_\alpha^\infty e^{-(\nu - 1)^2/3} \d\nu \right) \\
	&\le \frac{8(t-1)}{p\Delta}\frac{3\alpha}{\alpha-1} e^{-(\alpha-1)^2/3},
\end{align*}
where we define $\alpha = \sqrt{ap\Delta/8(t-1)}$.  
Using this and the result of Lemma \ref{lem:2ndmom} in \myeqref{eqn:1stmom1}
\begin{align*}
	\E{\norm{\htheta_t}{}} &\ge \left(\frac{C(\gamma, \Delta)}{\alpha}\sqrt\frac{\Delta}{8} %
	- \frac{3\sqrt{8}}{(\alpha-1)\sqrt\Delta}e^{-(\alpha-1)^2/3}\right) \sqrt\frac{t-1}{p} \\
	&\ge \left(\frac{C(\gamma, \Delta)}{\alpha}\sqrt\frac{\Delta}{8} %
	- \frac{6}{\alpha}\sqrt\frac{8}{\Delta}e^{-(\alpha-1)^2/3}\right) \sqrt\frac{t-1}{p},
\end{align*}
where the last inequality holds when $\alpha \ge 2$. Using 
$\alpha(\gamma, \Delta) = 1 + [3\log( 96/\Delta C(\gamma, \Delta))]^{1/2} >2$, 
the second term in leading constant is half that of the first, and we get 
the desired result.

\section{Proof of Theorem \ref{thm:optrewlar}}\label{sec:Proof2}

We now consider the large time horizon of $t > p\Delta$ for strategy
\BallStrategy, assuming the special case $\cX_p=\Ball(1)$. 
Throughout,  we will adopt 
the notation $\bbeta_t^2 = 1-\beta_t^2$.
To begin, we
bound the mean squared error in estimating $\theta$ using the following

\begin{lemma}[Upper bound on Squared Error]
	Under the conditions of Theorem \ref{thm:optrewlar} we have $\forall \, t \ge p\Delta + 1$:
	\begin{align*}
		\E(\Tr(\Sigma_t)) &\leq C_4(\Delta) \sqrt\frac{p}{t},
	\end{align*}
	where $C_4(\Delta) \equiv 3(\Delta+1)/\sqrt{\Delta}$.
	\label{lem:tracebnd}
\end{lemma}
\begin{proof}
As $\Sigma_t = (\Sigma_{t-1}^{-1} + \frac{1}{\sigma^2}x_tx_t^\sT)^{-1}$, 
we use the inversion lemma to get:
\begin{align*}
	\Tr(\Sigma_t) &= \Tr(\Sigma_{t-1}) - \frac{x_t^\sT\Sigma_{t-1}^2x_t}{\sigma^2 + x_t^\sT\Sigma_{t-1}x_t} \\
	&\leq \Tr(\Sigma_{t-1}) - \frac{p}{(\Delta + 1)}x_t^\sT\Sigma_{t-1}^2x_t, 
\end{align*}
where the inequality follows from $\Sigma_{t-1} \mle \id_p/p$ and 
$\norm{x_t}^2 \le 1$ for each $\ell$. Using $x_t = \bbeta_t 
\frac{\hth_t}{\norm{\hth_t}} + \beta_t \proj_t u_t$ and taking 
expectations on either side, we obtain:
\begin{align*}
	\E(\Tr(\Sigma_t)) &\leq \E(\Tr(\Sigma_{t-1})) - \frac{p}{\Delta+1}%
	\left[ \left( \bbeta_t^2 - \frac{\beta_t^2}{p} \right)\frac{\htheta_t^\sT \Sigma_{t-1}^2\htheta_t}{\norm{\htheta_t}^2}%
	+ \frac{\beta_t^2}{p}\E(\Tr(\Sigma_{t-1}^2)) \right] \\
	&\le \E(\Tr(\Sigma_{t-1})) - \frac{\beta_t^2}{\Delta + 1} \E(\Tr(\Sigma_{t-1}^2)),
\end{align*}
where we used $\bbeta_t^2 - \beta_t^2/p \ge 0$. This follows because
$\beta_t^2 \le 2/3 \le p/(p+1)$ when $t \ge p\Delta$ and $p \ge  2$. 
Employing Cauchy-Schwartz twice and using substituting for $\beta_t^2$
we get the following recursion in $\E(\Tr(\Sigma_t))$:
\begin{align}
	\E(\Tr(\Sigma_t)) \leq \E(\Tr(\Sigma_{t-1})) - \frac{2\sqrt\Delta}{3(\Delta +1)} \frac{1}{\sqrt{pt}} [\E(\Tr(\Sigma_{t-1}))]^2.
	\label{ineq:mainrec}
\end{align}
 The function $f(z) = z - z^2/b$ is increasing $z$ when 
$z \in (0, b/2)$. For the recursion above:
\begin{align*}
	b = b(t) &= \frac{3}{2}\sqrt\frac{ pt}{\Delta}(\Delta+1)  \\
	&> p(\Delta+1) \\
	&\geq 4, 
\end{align*}
since $p\Delta\geq 2$ and $p\geq 2$.  Also, we know that $\Sigma_t 
\mle \id_p/p$ and hence $\Tr(\Sigma_t) \leq 1$ 
with probability 1 and that $\E(\Tr(\Sigma_t))$ is decreasing in 
$t$. Thus the right hand side of the 
recursion is increasing in its argument. A standard induction 
argument then implies that $\E(\Tr(\Sigma_t))$ is bounded pointwise
by the solution to the following equation:
\begin{align*}
	y(t) &= y(t_0) - c\int_{t_0}^t \frac{y^2(s)}{\sqrt s} \mathrm{d}s, 
\end{align*}
with the initial condition $t_0 = p\Delta, \; y(t_0) = 1$, where 
$c = 2\sqrt{\Delta}/3 (\Delta+1)\sqrt p$. The solution is 
explicitly computed to yield:
\begin{align*}
	\E(\Tr(\Sigma_t)) \le \left[1+\frac{c'}{2}\left(\sqrt\frac{t}{p} - \sqrt{\Delta}\right)\right]^{-1}, 
\end{align*}
where $c' = c\sqrt{p} = 2\sqrt\Delta/3(\Delta+1)$. Since
the constant term is always positive, we can remove it and obtain the 
required result.

\end{proof}

We can now prove the following result: 
\begin{lemma}[]
	For all $t > p\Delta$, under the conditions of Theorem \ref{thm:optrewlar}:
	\begin{align*}
		\E\left[ \theta^\sT\left( \frac{\theta}{\norm\theta} - \frac{\hth_t}{\norm{\hth_t}} \right) \right] %
		&\leq 12(\Delta+1)\sqrt\frac{e}{\Delta}  \left(\frac{p}{t}\right)^{1/2 - 1/2(p+2)}.
	\end{align*}
	
	\label{lem:splitbnd}
\end{lemma}
\begin{proof}
Using the linearity of expectation:	
\begin{multline} \label{ineq:regrbnd1}
	\E\left[\theta^\sT\left( \frac{\theta}{\norm\theta} - \frac{\hth_t}{\norm{\hth_t}}\right)\right] %
	\leq \E\left[\theta^\sT\left( \frac{\theta}{\norm\theta} - \frac{\hth_t}{\norm{\hth_t}}\right)\ind(\norm{\theta} < \eps )\right] %
	 + \E\left[\theta^\sT\left( \frac{\theta}{\norm\theta} - \frac{\hth_t}{\norm{\hth_t}}\right) \ind(\norm{\theta}\geq \eps) \right].
	\end{multline}
We bound the first term as follows:
\begin{align*}
	\E\left[\theta^\sT\left( \frac{\theta}{\norm\theta} - \frac{\hth_t}{\norm{\hth_t}}\right)\ind(\norm{\theta} < \eps )\right] %
	&\leq \E\left[ \norm{\theta}\left\lVert\frac{\theta}{\norm\theta} - \frac{\hth_t}{\norm{\hth_t}} \right\rVert \ind(\norm\theta \leq \eps )\right] \\
	&\leq 2\eps \P(\norm\theta \leq \eps) \\
	&\leq 2\eps^{p+1}e^{p/2}.
\end{align*}
The first inequality is Cauchy-Schwartz, the second follows from bounds on the norm
of either vectors while the third is a standard Chernoff bound computation using the
fact that $\theta \sim \normal(0, \id_p/p)$. 
The second term can be bounded as follows:
\begin{align*}
	\E\left[\theta^\sT\left( \frac{\theta}{\norm\theta} - \frac{\hth_t}{\norm{\hth_t}}\right) \ind(\norm{\theta}\geq \eps) \right] %
	&\leq \E\left(\frac{2\norm{\theta - \hth_t}^2}{\norm{\theta}} \ind(\norm{\theta} \geq \eps)\right) \\
	&\leq \frac{2}{\eps}\E(\norm{\theta - \hth_t}^2) \\
	&\leq \frac{2}{\eps}\E\left(\Tr\Sigma_t\right).
\end{align*}
The first inequality follows from Lemmas 3.5 and 3.6 of \cite{RusTsi10}, 
the second follows from the fact that $\norm{\theta -\hth_t}^2$ is 
nonnegative and the indicator is used. Combining the bounds above and 
Lemma \ref{lem:tracebnd} we get:
\begin{align*}
	\E\left[\theta^\sT\left( \frac{\theta}{\norm\theta} - \frac{\hth_t}{\norm{\hth_t}}\right)\right] 
	&\leq 2\eps^{p+1}e^{p/2} + \frac{2C_4(\Delta)}{\eps}\sqrt\frac{p}{t}.
\end{align*}
Optimizing over $\eps$ we obtain:

\begin{align*}
	\E\left[\theta^\sT\left( \frac{\theta}{\norm\theta} - \frac{\hth_t}{\norm{\hth_t}} \right) \right] %
	&\leq 4\left(C_4(\Delta)e^{1/2} \sqrt\frac{p}{t}\right)^{1 - 1/(p+2)} \\
	&\leq 4\,e^{1/2}C_4(\Delta) \left(\frac{p}{t}\right)^{1/2 - 1/2(p+2)} .
\end{align*}

\end{proof}
We using Lemma \ref{lem:splitbnd} we can now prove Theorem \ref{thm:optrewlar}
for the large time horizon. Let $\rho_t$ denote the expected regret incurred by \OurStrategy at time
$t > p\Delta$. By definition, we write it as:
\begin{align*}
	\rho_t &= \E\left[{\theta^\sT\left(\frac{\theta}{\norm{\theta}{}} - \bbeta_t\frac{\hth_t}{\norm{\hth_t}} - \beta_t \proj_t u_t\right)}\right] \\
 	&= \E\left[ \theta^\sT\left(\frac{\theta}{\norm{\theta}{}} - \bbeta_t\frac{\hth_t}{\norm{\hth_t}} \right)\right], 
\end{align*}
as $u_t$ is zero mean conditioned on past observations. We split the 
first term in two components to get:
\begin{align*}
	\rho_t \leq (1-\bbeta_t)\E\norm{\theta} + \bbeta_t \E\left[ \theta^\sT\left(\frac{\theta}{\norm{\theta}{}} - \frac{\hth_t}{\norm{\hth_t}{}}\right) \right]
\end{align*}

We know that $0 \le 1 - \bbeta_t \leq \beta_t^2 = \sqrt{4p\Delta/9t}$. We use this and the 
result of Lemma \ref{lem:splitbnd} to bound the right hand side above as:
\begin{align*}
	\rho_t \leq \frac{2}{3}\left(\frac{p\Delta}{t}\right)^{1/2} +  12(\Delta + 1)\sqrt\frac{e}{\Delta} \left(\frac{p}{t}\right)^{1/2 - \omega(p)}, 
\end{align*}
where we define $\omega(p) \equiv 1/(2(p+2))$. 
Summing over the relevant interval and bounding by the corresponding 
integrals, we obtain:
\begin{align*}
	\sum_{\ell = p\Delta+1}^{t} \rho_\ell %
	&\leq \frac{4\sqrt\Delta}{3}(pt)^{1/2} + 24(\Delta+1)\sqrt\frac{e}{\Delta} (pt)^{1/2 + \omega(p)} \\
	&\leq C_3(\Delta) (pt)^{1/2 + \omega(p)}, 
\end{align*}
where $C_3(\Delta) = 4\sqrt\Delta/3 + 24(\Delta+1)\sqrt e/\sqrt\Delta$ 
and $\omega(p) =1/2(p+2)$. We can use 
$C_3(\Delta) \equiv 70(\Delta+1)/\sqrt\Delta$ for simplicity. 

\section*{Acknowledgments}

This work was partially supported by the NSF CAREER award CCF-0743978, the NSF grant DMS-0806211, and the AFOSR grant
FA9550-10-1-0360.

\appendix

\section{Properties of the set of arms}
\label{app:Arms}

\subsection{Proof of Lemma \ref{lemma:ContainsBall}}
We let $\cX'_p = \partial\Ball(\rho/\sqrt{3})$, 
where $\partial S$ denotes the boundary of a set $S$. 
For each $x \in \cX'_p$ denote the projection orthogonal to
it by $\proj_x$. We use the distribution 
$\P_x(z)$ induced by:
\begin{align*}
	z &= x + \sqrt\frac{2}{3}\rho\,\proj_x u,
\end{align*}
where $u$ is chosen uniformly at random on the unit sphere. This 
distribution is in fact supported on $\Ball(\rho) \subseteq \cX_p$. 
Also, we have, for all $x \in \cX'_p$,  $\E_{{x}} (z) = x $. 
Computing the second moment: 
\begin{align*}
	\E_{x}(zz^\sT) &= \E\left( xx^\sT + \frac{2\rho^2}{3} \proj_x uu^\sT \proj_x \right) \\
	&= xx^\sT + \frac{2\rho^2}{3p} \proj_x \\
	&= \frac{2\rho^2}{3p} \id_p + \left( 1-\frac{2}{p} \right)xx^\sT \\
	&\mge \frac{2\rho^2}{3p}\id_p,
\end{align*}
where in the first equality we used linearity of expectation, 
and that the projection mapping is idempotent. This yields 
$\gamma = 2\rho^2/3$. Since $\cX'_p = \partial\Ball(\rho/\sqrt 3)$ we 
obtain $\kappa = \inf_{\{\theta : \norm{\theta} = 1\}} \sup_{\{x \in \cX'_p\}} \< \theta, x\> = \rho/\sqrt 3$. 
Thus this construction satisfies Assumption \ref{assm:1}. 
Note the fact that \BallStrategy is a special case of 
\OurStrategy follows from the fact that we can use $\rho=1$
above when $\cX_p = \Ball(1)$. 

\subsection{Proof of Proposition \ref{prop:HullContainsBall}}

Throughout we will denote by $\conv(S)$ the convex hull of set $S$,
and by $\cconv(S)$ its closure. Also, it is sufficient to consider
Assumption \ref{assm:1}.2 for $\|\theta\|=1$.

It is immediate to see that $\Ball(\kappa)\subseteq \cconv(\cX'_p)$
implies Assumption  \ref{assm:1}.2. Indeed 
\begin{align*}
\sup \big\{\<\theta,x\> :\;  x\in\cX'_p\big\}
&= \sup \big\{\<\theta,x\> :\;  x\in\conv(\cX'_p)\big\}\\
&=\max \big\{\<\theta,x\> :\;  x\in\cconv(\cX'_p)\big\}\\
&\ge \max \big\{\<\theta,x\> :\;  x\in\Ball(\kappa)\big\}\ge
\kappa\|\theta\|\, ,
\end{align*}
where the last inequality follows by taking $x=\kappa
\theta/\|\theta\|$.

In order to prove the converse, let 
\begin{align*}
\kappa_0\equiv \sup\Big\{\rho:\; \Ball(\rho)\in\conv(\cX'_p)\Big\}\, .
\end{align*}
We then have $\Ball(\kappa_0)\subseteq \cconv(\cX'_p)$. 
Assume by contradiction that $\kappa_0<\kappa$. Then there exists at
least one point $x_0$ on the boundary of $\cconv(\cX'_p)$ such that
$\|x_0\|=\kappa_0$ (else $\kappa_0$ would not be the
supremum). 

By the supporting hyperplane theorem, there exists a closed half space $\cH$
in $\reals^p$ such that $\cconv(\cX'_p)\subseteq \cH$ and $x_0$ is on
the boundary $\partial\cH$ of $\cH$. It follows that
$\Ball(\kappa_0)\subseteq \cH$ has well, and therefore $\partial\cH$
is tangent to the ball at $x_0$. Summarizing
\begin{align}
\cconv(\cX_p') \subseteq \cH \equiv\Big\{\, x\in\reals^p\,
:\<x,x_0\>\le \kappa_0\|x_0\|\,\Big\}\, .
\end{align}
By taking $\theta=x_0/\|x_0\|$, we then have, for any
$x\in\cconv(\cX'_p)$, $\<\theta,x\>\le\kappa_0<\kappa$, which is
in contradiction with Assumption \ref{assm:1}.2.

\subsection{Proof of Proposition \ref{prop:UniformCloud}}
\subsubsection{Proof of condition 1}
Choose $\cX'_p = \cX_p\cap\Ball(\rho)$. We first prove that 
$f(\theta) \equiv \max_{x \in \cX'_p} \< \theta, x\>$ is Lipschitz
continuous with constant $\rho$. Then, employing an $\upsilon$-net argument, 
we prove that this choice of $\cX'_p$ satisfies Assumption \ref{assm:1}.1
with high probability.

Let $f(\theta_i) = \<\theta_i, x_i\>$ for $i = 1, 2$. Without loss of
generality, assume $f(\theta_1)>f(\theta_2)$. We then have: 
\begin{align*}
|f(\theta_1) - f(\theta_2)| &= |\<\theta_1, x_1\> - \<\theta_2, x_2\>| \\
&=|\<\theta_1, x_1\> - \<\theta_2, x_1\> +\<\theta_2, x_1\> - \<\theta_2, x_2\>|\\
&\le |\<\theta_1 - \theta_2, x_1\>|\\
&\le \norm{x_1}\norm{\theta_1-\theta_2} \\
&\le \rho\norm{\theta_1 - \theta_2}, 
\end{align*}
where the first inequality follows since $x_2$ maximizes $\<\theta_2, x_2\>$, the
second is Cauchy-Schwarz and the third from the fact that $x_1 \in \cX_p\cap\Ball(\rho)$. 

Since $f(\theta) = \norm{\theta}f(\theta/\norm{\theta})$, it suffices to consider
$\theta$ on the unit sphere $S_p$. Suppose $\Upsilon$ is an $\upsilon$-net  of the unit sphere, i.e.
a maximal set of points that are separated from each other by at least $\upsilon$. We can bound 
$|\Upsilon|$ by a volume packing argument: consider balls of radius $\upsilon/2$ around every 
point in $\Upsilon$. Each of these is disjoint (by the property of an $\upsilon$-net) and, by the
triangle inequality, are all contained in a ball of radius $1+\upsilon/2$. The latter has a volume
of $(1+ 2\upsilon^\inv)^p$ times that of each of the smaller balls, thus yielding that 
$|\Upsilon| \le (1+ 2\upsilon^\inv)^p$.

Now, $|\cX'_p|$ is binomial with mean $M\rho^p$ and variance
$M\rho^p(1-\rho^p)$. Consider a single point $\theta \in \Upsilon$. Due to rotational
invariance we may assume $\theta = e_1$, the first canonical basis vector. Conditional
on the event $E_n=\{\omega:\; |\cX'_p| = n\}$, the arms in $\cX'_p$ are uniformly distributed in $\Ball(\rho)$. 
Thus we have (assuming $z>0$):
\begin{align}
	\P(\max_{x\in\cX'_p}\, \<x, e_1\> \le z\rho|E_n) &= \prod_{j = 1}^n \P(\<x_j, e_1\> \le z\rho|E_n) \nonumber \\
	&= \left(\P(\<x_1, e_1\> \le z
          \rho|x_1\in\Ball(\rho)\right)^n\\
	&= \left(\P(\<x_1, e_1\> \le z\right)^n\\
,\label{eqn:uc1}
\end{align}
since the $\<x_j, e_1\>$, $j\in\{1,\dots,n\}$ are iid, and the
conditional distribution of $x_1$ given $x\in\Ball(\rho)$ is the same as
the unconditional distribution of $\rho\,x$. Let $Y_1 \cdots Y_p \sim \normal(0, 1/2)$ be iid and 
$Z \sim \text{Exp}(1)$ be independent of the $Y_i$. Then by Theorem 1 of \cite{Barthe05} 
$\<x_1, e_1\>$ is distributed as $\rho Y_1/(\sum_{i=1}^p Y_i^2 + Z)^{1/2}$. By a 
standard Chernoff argument, $\P(\sum_{i = 2}^p Y_i^2 \ge 2(p-1)) \le \exp\{-c(p-1)\}$
where $c = (\log 2 - 1)/2$. Also, $\P(Z \ge p) = \exp(-p)$ and $\P(Y_1^2 \ge p ) \le  2\exp(-p)$.
This allows us the following bound:
\begin{align*}
	\P(\<x_1, e_1\> \le z) &= \P\left(  \frac{Y_1}{\sum_{i=1}^p Y_i^2 + Z}  \le z \right) \\
	&\le \nu(p) + (1- \nu(p))\P\left( \frac{Y_1}{\sqrt{4p-2}} \le z \Big\vert Y_1^2 \le p  \right),
\end{align*}
where $\nu(p) \equiv 3\exp(-p) + \exp(-c(p-1))$. We further simplify to obtain:
\begin{align*}
	\P(\<x_1, e_1\> \le z) &\le 1 - (1-\nu(p)) \P\left(\frac{Y_1}{\sqrt{4p-2}} \ge z \Big\vert Y_1^2 \le p\right) \\
	&\le 1 - (1-\nu(p))\left( {F_G(\sqrt{2p}) - F_G(z\sqrt{8p-4})} \right),
\end{align*}

and $F_G(\cdot)$ denotes the Gaussian 
cumulative distribution function.
Employing this in \myeqref{eqn:uc1}:
\begin{align*}
	\P(\max_{x\in\cX'_p} \<x, e_1\> \le z\rho|E_n) &\le \left[ 1 - (1-\nu(p))\left( {F_G(\sqrt{2p}) - F_G(z\sqrt{8p-4})} \right) \right]^n \\
	&\le \exp\left[ -n (1-\nu(p)) (F_G(\sqrt{2p}) - F_G(z\sqrt{8p-4})\right].
\end{align*}

For $p\ge6$, we have that $1-\nu(p) \ge 1/2$ and $F_G(\sqrt{2p}) - F_G(\sqrt{8p-4}) \ge 3^{-p}/2$.
Using this, substituting $z = 1/2$ and that $|\cX'_p| \ge M\rho^p/2$ with probability 
at least $1- \exp(-M\rho^p/8)$ we now have:
\begin{align*}
	\P(\max_{x\in \cX'_p} \<x, e_1\> \le \rho/2) \le \exp(-M\rho^p 3^{-p}/4) + \exp(-M\rho^p/8) 
\end{align*}
We may now union bound over $\Upsilon$ using rotational invariance 
to obtain:

\begin{align*}
	\P(\min_{\theta \in \Upsilon} \max_{x\in\cX'_p} \<x, \theta\> \le \rho/2) \le (1+2\upsilon^\inv)^p(\exp(-M\rho^p 3^{-p}/4) + \exp(-M\rho^p/8))
\end{align*} 
Using $\rho = 1/2$, $\upsilon = 1/2$, $M = 8^p$ and that $f(\theta)$ is
Lipschitz, we then obtain:
\begin{align*}
	\P(\min_{\norm{\theta} = 1} \max_{\cX'_p} \<x, \theta\> \le 1/4) &\le 5^p[\exp(-4^{p-1}/3^p) + \exp(-4^p/8)] \\
	&\le \exp(-p),
\end{align*}
when $p \ge 20$. 
%
%
\subsubsection{Proof of condition 2}
Fix radii $\rho$ and $\delta$ such that $\rho+\delta \le 1$. We choose
the $\cX_p'$ to be $\cX_p \cap \Ball(\rho)$. Consider a
point $x$ such that $\norm{x} \le 1- \delta$. We consider the events $E_i, D_i$:
\begin{align*}
	E_i &\equiv \{\nexists \textrm{ a distribution } \P_{x_i} \textrm{ satisfying %
	Assumption \ref{assm:1}.2} \} \\
	D_i &\equiv \{x_i \in \Ball(\rho)\}
\end{align*}
We now bound $\P(E_i | D_i)$. Within a distance $\delta$ around $x_i$,
there will be, in expectation, $M\delta^p$ arms (assuming the total number 
of points is $M+1$). Indeed the distribution of the number of arms within distance
$\delta$ around $x_i$ is binomial with mean $M\delta^p$ and variance $M\delta^p(1-\delta^p)$.

Conditional on the number of arms in $\Ball(\delta, x_i)$ being $n$, these arms are uniformly
distributed in $\Ball(\delta, x_i)$ and are independent of $x_i$. 
We will use $\prob_n$ to denote this conditional probability measure. 
Denote the
arms within distance $\delta$ from $x_i$ to be $v_1, v_2 \ldots v_n$. Define 
$u_j \equiv v_j - x_j$, $\bar u \equiv (\sum_{j=1}^n u_j)/n$ for all $j$
and $Q = (\sum_{j = 1}^n u_j u_j ^\sT)/n$. To construct
the probability distribution $\P_{x_j}$, we let the weight $w_j$ on the arm $v_j$ to be:
\begin{align*}
	w_j &= \frac{1}{n}\left(\frac{1 - u_j^\sT Q^{-1}\bar u}{1 - \bar u ^\sT Q^{-1} \bar u}\right)
\end{align*}
It is easy to check that these weights yield the correct
first moment, i.e. $\sum_{j=1}^n w_jv_j = x_i$. Before 
considering the second moment, we first show that $Q$ concentrates 
around its mean. It is straightforward to compute that
$\E(Q) = \E(u_1u_1^\sT) = \mu\id_p$, where $\mu = \delta^2/(p+2)$.
By the matrix Chernoff bound \cite{AW02,Tro10}, there exist $c>0$
such that:
\begin{align}\label{eqn:matcher}
	\P_n(\norm{Q^\inv} \ge \frac{2}{\mu}) &\le p\exp(-cn\mu/\delta^2), 
\end{align}
where $\norm{Q}$ denotes the operator norm and the probability 
is over the distribution of the $u_j$. 
We further have, for all $j$:
\begin{align}\label{eqn:Exbnd}
	w_j \ge \frac{1}{n}\left( 1-\norm{u_j}\norm{Q^\inv \bar u} \right)
	&\ge \frac{1}{n} \left( 1 - \delta \norm{Q^\inv}\norm{\bar u}\right).
\end{align}
Also, using Theorem 2.1 of \cite{JudNem08} we obtain that:
\begin{align*}
	\P_n(\norm{\bar u} \ge \delta/n^{1/4} ) &\le \exp\left\{-\frac{(n^{1/4} -1)^2}{2}\right\} \\
	&\le \exp(-n^{1/2}/4),
\end{align*}
for $n \ge 16$. 
Combining this with \myeqref{eqn:matcher} and continuing inequalities in
\myeqref{eqn:Exbnd}, we obtain, for all $j$,:
\begin{align}\label{eqn:wibnd}
	 w_j \ge \frac{1}{n} \left(1 - \frac{2(p+2)}{n^{1/4}}\right),
\end{align}
with probability at least $1 - \omega(n, p)$ where $\omega(n, p) = p\exp(-cn/2p)+ \exp(-n^{1/2})$.
We can now bound the second moment of $\P_x$:

\begin{align*}
	\sum_{j= 1}^n w_j v_j v_j^\sT &= \sum_{j = 1}^n w_j u_j u_j^\sT + xx^\sT \\
	&\mge \sum_{j=1}^{n} w_j u_j u_j^\sT \\
	&\mge \left(\frac{1}{2} - \frac{p+2}{n^{1/4}}\right) \frac{\delta^2}{p+2}\id_p
\end{align*}
where the last inequality holds with probability at least $ 1-\omega(n, p)$. Thus
we can obtain $\gamma = \delta^2/8$ for $n \ge [4(p+2)]^4 $. 

In addition, a standard Chernoff bound argument yields that the number of arms
in $\Ball(\delta, x_i)$ is at least $M\delta^p/2$ with probability at least
$1 - \exp(-M\delta^p/8)$. With this, we can bound $\P(E_i\vert D_i)$: 
\begin{align*}
	\P(E_i|D_i) \le \exp(-M\delta^p/8) + \omega(M\delta^p/2, p).
\end{align*}

The event $F$ that the uniform cloud 
does not satisfy Assumption \ref{assm:1}.2 can now be decomposed as follows:

\begin{align*}
	\P(F) &= \P\left( \bigcup_{i = 1}^M (E_i\cap D_i) \right) \\
	&\le \sum_{i = 1}^{M+1} \P(E_i | D_i)\P(D_i) \\
	&\le 2M\rho^p\{\exp(-M\delta^p/8) + \omega(M\delta^p/8, p)\}
\end{align*}

Choosing $\delta=\rho=1/2$, with $M = 8^p$, we get that the uniform cloud
satisfies Assumption \ref{assm:1}.2 with $\gamma \ge \delta^2/8 = 1/32$ with 
probability at least $ 1 - 2\cdot4^p[\exp(-4^p/8) + \omega(4^p/8, p)] \ge 1 - \exp(-p) $
when $p \ge 10$. 
Summarizing the proofs of both conditions we have, choosing the number of points
$M = 8^p$, the subset $\cX'_p = \cX_p \cap \Ball(1/2)$, we obtain constants
$\kappa = 1/4$ and $\gamma = 1/32$ with probability at least $1 - 2\exp(-p)$, provided
$p \ge 20$. 

\bibliographystyle{alpha}

\begin{thebibliography}{BGMN05}

\bibitem[Aue02]{Auer02}
P.~Auer.
\newblock Using confidence bounds for exploitation-exploration trade-offs.
\newblock {\em Journal of Machine Learning Research}, 3:2002, 2002.

\bibitem[AW02]{AW02}
R.~Ahlswede and A.~Winter.
\newblock Strong converse for identification via quantum channels.
\newblock {\em IEEE Trans. Inform. Theory}, 48(3):569 -- 579, March 2002.

\bibitem[AYPS11]{AbSze11}
Y.~Abbasi-Yadkori, D.~P{\'a}l, and C.~Szepesv{\'a}ri.
\newblock Improved algorithms for linear stochastic bandits.
\newblock In {\em NIPS}, pages 2312--2320, 2011.

\bibitem[BGMN05]{Barthe05}
F.~Barthe, O.~Gu{\'e}don, S.~Mendelson, and A.~Naor.
\newblock A probabilistic approach to the geometry of the lpn-ball.
\newblock {\em The Annals of Probability}, 33(2):480--513, 2005.

\bibitem[BK07]{BK07long}
R.~M. Bell and Y.~Koren.
\newblock Scalable collaborative filtering with jointly derived neighborhood
  interpolation weights.
\newblock In {\em ICDM '07: Proceedings of the 2007 Seventh IEEE International
  Conference on Data Mining}, pages 43--52, Washington, DC, USA, 2007. IEEE
  Computer Society.

\bibitem[CR09]{CaR08}
E.~J. Cand{\`e}s and B.~Recht.
\newblock Exact matrix completion via convex optimization.
\newblock {\em Foundation of computational mathematics}, 9(6):717--772,
  February 2009.

\bibitem[CT10]{CandesTaoMatrix}
E.J. Cand{\`e}s and T.~Tao.
\newblock The power of convex relaxation: Near-optimal matrix completion.
\newblock {\em Information Theory, IEEE Transactions on}, 56(5):2053--2080,
  2010.

\bibitem[DHK08]{Dani08}
V.~Dani, T.P. Hayes, and S.M. Kakade.
\newblock Stochastic linear optimization under bandit feedback.
\newblock In {\em COLT}, pages 355--366, 2008.

\bibitem[DM13]{OurImplementationInteractive}
Y.~Deshpande and A.~Montanari.
\newblock Implementing an interactive recommendation system.
\newblock In preparation, 2013.

\bibitem[Gro09]{Gross09}
D.~Gross.
\newblock Recovering low-rank matrices from few coefficients in any basis.
\newblock {\sf arXiv:0910.1879}, 2009.

\bibitem[JN08]{JudNem08}
A.~Juditsky and A.~Nemirovski.
\newblock {Large Deviations of Vector-Valued Martingales in 2-Smooth Normed
  Spaces}.
\newblock {\sf arXiv:0809.0813}, 2008.

\bibitem[KBV09]{KBV09}
Y.~Koren, R.~Bell, and C.~Volinsky.
\newblock Matrix factorization techniques for recommender systems.
\newblock {\em Computer}, 42(8):30--37, August 2009.

\bibitem[KLT11]{KoltchinskiiMatrixCompletion}
V.~Koltchinskii, K.~Lounici, and A.B. Tsybakov.
\newblock Nuclear-norm penalization and optimal rates for noisy low-rank matrix
  completion.
\newblock {\em Ann. Statist.}, 39:2302--2329, 2011.

\bibitem[KMO10a]{KMO09}
R.~H. Keshavan, A.~Montanari, and S.~Oh.
\newblock Matrix completion from a few entries.
\newblock {\em IEEE Trans. Inform. Theory}, 56(6):2980--2998, June 2010.

\bibitem[KMO10b]{KMO09noise}
R.~H. Keshavan, A.~Montanari, and S.~Oh.
\newblock Matrix completion from noisy entries.
\newblock {\em J. Mach. Learn. Res.}, 11:2057--2078, July 2010.

\bibitem[Kor08]{Kor08long}
Y.~Koren.
\newblock Factorization meets the neighborhood: a multifaceted collaborative
  filtering model.
\newblock In {\em KDD '08: Proceeding of the 14th ACM SIGKDD international
  conference on Knowledge discovery and data mining}, pages 426--434, New York,
  NY, USA, 2008. ACM.

\bibitem[RT10]{RusTsi10}
P.~Rusmevichientong and J.N. Tsitsiklis.
\newblock Linearly parameterized bandits.
\newblock {\em Math. Oper. Res.}, 35(2):395--411, 2010.

\bibitem[SJ03]{SJ03}
N.~Srebro and T.~Jaakkola.
\newblock Weighted low-rank approximations.
\newblock In {\em 20th International Conference on Machine Learning}, pages
  720--727. AAAI Press, 2003.

\bibitem[SPUP02]{ScheinColdStart}
A.I. Schein, A.~Popescul, L.H. Ungar, and D.M. Pennock.
\newblock Methods and metrics for cold-start recommendations.
\newblock In {\em Proceedings of the 25th annual international ACM SIGIR
  Conference on Research and Development in Information Retrieval}, pages
  253--260, 2002.

\bibitem[SRJ05]{SRJ05}
N.~Srebro, J.~D.~M. Rennie, and T.~S. Jaakola.
\newblock Maximum-margin matrix factorization.
\newblock In {\em Advances in Neural Information Processing Systems 17}, pages
  1329--1336. MIT Press, 2005.

\bibitem[SW06]{SlaneyDiversity}
M.~Slaney and W.~White.
\newblock Measuring playlist diversity for recommendation systems.
\newblock In {\em Proceedings of the 1st ACM workshop on Audio and music
  computing multimedia}, 2006.

\bibitem[Tro12]{Tro10}
J.A. Tropp.
\newblock User-friendly tail bounds for sums of random matrices.
\newblock {\em Foundations of Computational Mathematics}, 12(4):389--434, 2012.

\bibitem[ZH08]{ZhangDiversity}
M.~Zhang and N.~Hurley.
\newblock Avoiding monotony: improving the diversity of recommendation lists.
\newblock In {\em Proceedings of the 2008 ACM conference on Recommender
  Systems}, 2008.

\end{thebibliography}

\end{document}